\documentclass{svproc}
\usepackage{url}

\usepackage{cite}
\usepackage{url}

\usepackage{amsmath,amssymb,amsfonts}
\usepackage{graphicx}
\usepackage{textcomp}

\usepackage{enumitem}
\usepackage{multirow}
\usepackage{booktabs}
\usepackage[ruled,vlined,linesnumbered]{algorithm2e}
\usepackage{algorithmic,algorithm2e,float}

\SetKw{Continue}{continue}

\usepackage{lipsum}
\newcommand\blfootnote[1]{%
  \begingroup
  \renewcommand\thefootnote{}\footnote{#1}%
  \addtocounter{footnote}{-1}%
  \endgroup
}

\usepackage[hang,flushmargin]{footmisc}
\renewcommand{\thefootnote}{\fnsymbol{footnote}}

\usepackage{nicefrac}
\usepackage{subcaption}
\usepackage{lipsum}

\DeclareMathOperator*{\argmin}{arg\,min}
\DeclareMathOperator*{\argmax}{arg\,max}
\newcommand{\vect}[1]{\boldsymbol{#1}}

\newcommand{\be}{\begin{equation}}
\newcommand{\ee}{\end{equation}}

\newcommand{\C}{\mathcal{C}}

\newcommand{\N}{\mathcal{N}}
\newcommand{\W}{\mathcal{W}}

\newcommand{\w}{\vect{w}}
\newcommand{\e}{\vect{e}}

\newcommand{\wuser}{\vect{{w}}^*}
\newcommand{\f}{\vect{f}}

\newcommand{\rAbs}{r}

\newcommand{\Uni}{\mathtt{Uniform}}
\newcommand{\Greedy}{\mathtt{MRPS}}
\newcommand{\Random}{\mathtt{Random}}
\newcommand{\Regret}{\mathtt{Regret}}

\usepackage[dvipsnames]{xcolor}

\newcommand{\blue}[1]{\textcolor{black}{#1}}

\newcommand{\grey}[1]{\textcolor{gray}{#1}} 

\usepackage{amsthm}
\theoremstyle{definition}
\newtheorem{assumption}{Assumption}

\newtheorem{observation}{Observation}
\newtheorem*{remark*}{Remark}

\usepackage{breqn}
\urldef{\smith}\url{stephen.smith@uwaterloo.ca}
\urldef{\botros}\url{alexander.botros@uwaterloo.ca}
\urldef{\wilde}\url{nwilde@uwaterloo.ca}

\begin{document}
	\mainmatter  % start of a contribution
	%
% 	\title{Robust Sampling Approximation for Multi-objective Optimization}

% 	\title{Finding Representative Weights on  Multiple Objectives in Motion Planning Optimization Problems}
	
%\title{How to Weight Multiple Objectives in Motion Planning Optimization Problems}

\title{Error-Bounded Approximation of Pareto Fronts in Robot Planning Problems}

\titlerunning{Error-Bounded Approximation of Pareto Fronts in Robot Planning Problems} 
	% abbreviated title (for running head) also used for the TOC unless
	% \Soctitle is used
	\author{Alexander Botros\footnotemark[1]
	\and Armin Sadeghi\footnotemark[1]
	\and Nils Wilde\footnotemark[2]
	\and Javier Alonso-Mora\footnotemark[2]
	\and
	Stephen L. Smith\footnotemark[1]}

	\authorrunning{A.~Botros, A.~Sadeghi, N.~Wilde, J.~Alonso-Mora and S.~L.~Smith} % abbreviated author list (for running head)
	\tocauthor{Alexander Botros, Armin Sadeghi, Nils Wilde, Javier Alonso-Mora, Stephen L. Smith^1}
	%
	
	%\footnotetext{\textbf{Acknoledgements:} This research is partially supported by the Natural Sciences and Engineering Research Council of Canada (NSERC).}
	\institute{\footnotemark[1] Department of Electrical and Computer Engineering\\
University of Waterloo, Waterloo, ON, Canada\\
\texttt{\{alexander.botros, a6sadegh, stephen.smith\}@uwaterloo.ca},\\
\footnotemark[2] Department for Cognitive Robotics, 3ME\\
Delft University of Technology, Delft, Netherlands\\
\texttt{\{n.wilde, j.alonsomora\}@tudelft.nl},
}
	\maketitle % typeset the title of the contribution
	
	\blfootnote{\noindent\textbf{Acknowledgements:} This research is partially
    supported by the Natural Sciences and Engineering Research Council
    of Canada (NSERC), as well as  by the European Union's Horizon 2020 research and innovation program under Grant 101017008.}
   \blfootnote{Alexander Botros, Armin Sadeghi and Nils Wilde contributed equally.}
% 	\footnotetext[1]{Alexander Botros and Nils Wilde contributed equally.} 
    
%%%%%%%%%%%%%%%%%%%%%%%%%%%%%%%%%%%%%%%%%%%%%%%%%%%%%%%%%%%%    
% Abstract  
\begin{abstract}
    Many problems in robotics seek to simultaneously optimize several competing objectives under constraints. A conventional approach to solving such multi-objective optimization problems is to create a single cost function comprised of the weighted sum of the individual objectives. 
    Solutions to this scalarized optimization problem are Pareto optimal solutions to the original multi-objective problem. 
    However, finding an accurate representation of a Pareto front remains an important challenge. Using uniformly spaced weight vectors is often inefficient and does not provide error bounds. Thus, we address the problem of computing a finite set of weight vectors such that for any other weight vector, there exists an element in the set whose error compared to optimal is minimized. To this end, we prove fundamental properties of the optimal cost as a function of the weight vector, including its continuity and concavity. Using these, we propose an algorithm that greedily adds the weight vector least-represented by the current set, and provide bounds on the error. Finally, we illustrate that the proposed approach significantly outperforms uniformly distributed weights for different robot planning problems with varying numbers of objective functions.

\end{abstract}
	%

% %%%%%%%%%%%%%%%%%%%%%%%%%%%%%%%%%%%%%%%%%%%%%%%%%%%%%%%%%%%%
\section{Introduction}
Robot planning problems often face the challenge of simultaneously optimizing multiple objectives. Finding the appropriate trade-off between objectives remains a major challenge when deploying intelligent autonomous systems. For instance, autonomous cars optimize trajectories for passenger comfort and efficiency \cite{ziegler2014trajectory, botros2021tunable}, manipulators working in cooperative tasks consider human ergonomics as well as risk of collision or other damage in handover tasks \cite{bajcsy2017learning},
mobile robots performing transportation in industrial settings trade-off task efficiency and compliance with user specific norms \cite{wilde2020improving},
and autonomous mobility-on-demand systems seek to maximize service quality while minimizing operation cost \cite{cap2018multi}. 

\begin{figure*}[t!]
    \centering
    \begin{subfigure}[t]{0.4\textwidth}
        \centering
        \includegraphics[width=0.99\textwidth]{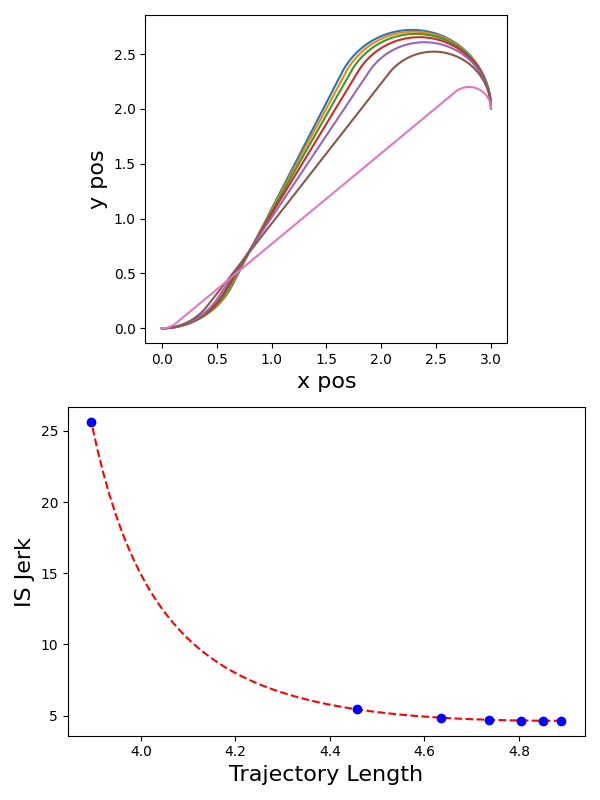}
        \caption{Uniform sampling.}
        \label{fig:Dubins_traject_Pareto_example_a}
    \end{subfigure}%
    ~ 
    \begin{subfigure}[t]{0.4\textwidth}
        \centering
        \includegraphics[width=0.99\textwidth]{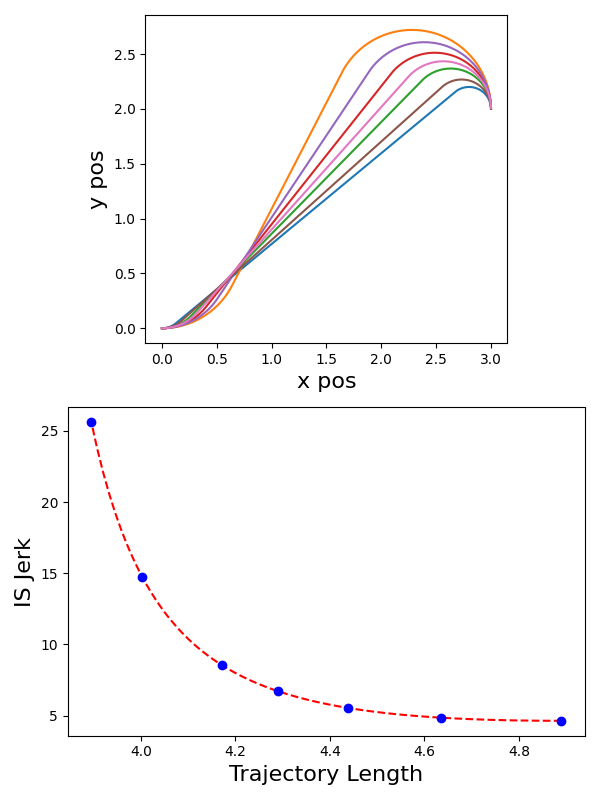}
        \caption{Proposed min-regret sampling.\bigskip}
        \label{fig:Dubins_traject_Pareto_example_b}
    \end{subfigure}
    \caption{\footnotesize \textit{Top:} Different optimal trade-offs between trajectory length and discomfort, sampled uniformly and by the proposed method. \textit{Bottom:} Approximated Pareto front (blue) and ground truth Pareto front (red).\normalsize}
    \label{fig:Dubins_traject_Pareto_examples}
\end{figure*}

A common technique in multi-objective optimization is to create a single cost function defined by the weighted sum of the individual objectives \cite{marler2010weighted}. This approach, called \emph{linear scalarization}, leads to solutions that are Pareto optimal for the multi-objective problem \cite{hwang2012multiple}. That is, the solution to the scalarized single-objective problem cannot be changed to improve the value of one of its objectives without degrading the value of another. This work is motivated by two classes of problems in robotics: 1) obtaining near optimal solutions to a linearly scalarized multi-objective optimization problem (LSMOP) for any given weight vectors, and 2) learning the preferred solution behavior of a human user. Applications of the first class include \cite{cap2018multi} in which an adjustable trade-off between the service quality and operating costs for autonomous mobility-on-demand systems are optimized. Conversely, the second class of problems seeks to compute a weight vector representing the relative importance of each objective to a user given some knowledge of that users' preferred solutions~\cite{sadigh2017active, zhang2019learning, holladay2016active, wilde2020improving, wilde2020active, biyik2019asking}. 

In either class, it is often beneficial to pre-compute solutions to the LSMOP for a set of weight vectors. If the LSMOP is computationally intensive to solve, requiring online solutions may be impractical. This motivates the problem of finding a set of weights and their corresponding optimal solutions such that for any possible weight vector, there exists an element of the set whose solution is close to optimal. A naive approach would involve densely sampling the set of all possible weight vectors. However, the sensitivity of some objectives to changes in solution may result in a skewed sampling of Pareto optimal solutions. This is illustrated in Figure \ref{fig:Dubins_traject_Pareto_examples} where we seek to compute a Dubins trajectory between fixed start and goal configurations that minimizes a trade-off between trajectory length and discomfort (measured as the integral of the squared jerk over the trajectory \cite{levinson2011towards}). In the left images, a set of weights vectors is selected uniformly and the resulting solution trajectories (top) and Pareto front (bottom) is shown. We observe that uniformly sampling over the set of possible weights generate a set of similar trajectories. In this paper, we propose a technique to construct a set of weight vectors such that their corresponding optimal solutions provide homogeneous coverage of the Pareto front (see Figure~\ref{fig:Dubins_traject_Pareto_example_b}).

%------------------------------------------------
%------------------------------------------------
\paragraph{Contributions: }The contributions of this work are fourfold: First, we propose an algorithm for homogeneous sampling of the Pareto front. Second, we prove fundamental properties of the optimal cost for any weight such as concavity and continuity.
% Third, leveraging our fundamental properties we provide a bound on the error incurred from the approximation of the optimal solution to LSMOP for any arbitrary weight with a solution corresponding to a weight vector from the computed set. 
\blue{Third, we show that using the best available solution from the computed set to approximate the optimal solution to LSMOP for any arbitrary weight has has a bounded error.}
Finally, we showcase the advantages of our sampling algorithm in different robotics applications, namely trajectory planning and reward learning.

%------------------------------------------------
\subsection{Related Work}
%------------------------------------------------
\paragraph{Linear Scalarization in Robotics:}
The simplicity of linear scalarization has made it one of the most widely used tools in robotic multi-objective optimization \cite{schutze2020Pareto}. Though the technique is not able to capture all Pareto-optimal solutions for non-convex fronts, it does guarantee that all LSMOP solutions are Pareto-optimal. In \cite{dolgov2010path}, a trajectory smoothing algorithm is proposed based on the weighted sum of the competing objectives: trajectory length, smoothness, and obstacle distances. The authors of \cite{zeng2021simultaneous} minimize the weighted trade-off between mission completion time and communication outage duration in the navigation of cellular-connected UAVs, while in \cite{9197201}, linear scalarization is used to optimize robotic limitations and observation rewards for use in autonomous human activity tracking. 

In human-robot interation (HRI), weight vectors are used to represent the relative importance of objectives -- often referred to as features -- to a user \cite{cakmak2011human, sadigh2017active, biyik2019asking, wilde2020active}. In \textit{reward learning} it is the objective to learn a user's weight vector using interactions such as demonstrations, corrections, or choice feedback.
In order to expedite the learning process, feasible solutions for the multi-objective optimization problem are often pre-computed and shown to the user who then provides feedback. In \cite{sadigh2017active, basu2018learning, biyik2019asking}, each pre-computed solution is generated with random action sequences. Thus, the solutions used in the learning process are usually not optimal for any weight. In \cite{wilde2020active,wilde2021learningScale}, the authors pre-compute Pareto-optimal solutions enabling an active learning method based on regret leading to significant improvement over randomly generated solutions. However, the work in \cite{wilde2020active, wilde2021learningScale} rely on uniformly sampled weight vectors. Though this approach asymptotically covers the set of all LSMOP-optimal solutions, it can be inefficient as different weight vectors can have very similar, or even identical solutions. The authors of \cite{bobu2020less} observed that the presence of similar solution strongly influences the Boltzmann decision model which is commonly used in HRI. They propose a decision model where a similarity metric corrects the bias induced by a high number of similar trajectories.
Similarly, in our work we are interested in finding solutions with dissimilar features. 
While \cite{bobu2020less} handles the over-representation of similar solutions in the ground set with their proposed decision model, we instead address the problem at an earlier stage: our algorithm can be used to generate a ground set where similarities are minimized.

%------------------------------------------------
\paragraph{Uniform Sampling of Pareto-Optimal Solutions:}
Uniform sampling of a Pareto front is a persistent problem in multi-objective optimization. Similar to our work, the authors of \cite{pereyra2013equispaced} offer a technique of Pareto-uniform sampling based on equispacing constraints. However, this work only considers the case of two competing objectives. Further, they accomplish their goal by solving a significantly harder problem than the original LSMOP. Also similar, the work in \cite{xu2021multi} proposes a set of weight vectors that approximately uniformly cover a Pareto front specifically for use in the design of robots. The authors design the set that minimizes the total squared error between the value of the objectives in the set and heuristic objectives. It therefore relies on the approximate optimality of these objectives. The work in \cite{parisi2017manifold} proposes a method to cover the set of Pareto-optimal solutions specifically for use in reinforcement learning applications. There, the authors seek to compute policies that maximize expected returns by computing, storing, and updating a set of samples. In \cite{schutze2020Pareto} the authors provide a means of exploring a (possibly non-convex) Pareto front in order to obtain a solution that is near-optimal for a user. That work starts with an initial guess solution and then moves in a direction according to the preferences of a user. While the convexity of the Pareto front (a requirement for linear scalarization to obtain all Pareto-optimal solutions) is not assumed, their technique does require solving the LSMOP online as the Pareto front is explored. Moreover, the requirement of a user-preferred direction is not assumed in our work. In \cite{lee2018sampling} a Pareto front approximation is proposed using Markov chain random walks. 
Their goal is to \emph{uniformly} place samples on the Pareto front, while our goal is to \emph{minimize error} in the space of Pareto-optimal costs.

%------------------------------------------------
\section{Problem Statement}
For $n\in\mathbb{N}$, a general multi-objective optimization problem (MOP) is of the form
\begin{equation*}
    \begin{split}
        \min_{s\in\mathcal{S}} \bigg(f_1(s), f_2(s),\dots,f_n(s)\bigg).
    \end{split}
\end{equation*}
Here, the set of feasible solutions given constraints is denoted $\mathcal{S}$, and it is desired to simultaneously minimize $n$ objectives $f_i(s), i=1,\dots, n$.  The linear scalarization of the MOP above would involve the creation of a single cost function by introducing a vector of weights $\w=[w_1,w_2,\dots,w_n]\in\mathbb{R}_{\geq 0}^n$. Let $c(s, \w)$ denote the cost of the solution $s$ evaluated by the weights $\w$, i.e., 
$$
    c(s, \w) = \sum_{i = 1}^{n}w_i \cdot f_i(s) = \w \cdot \f(s),
$$
where $\f(s)=[f_1(s),\dots,f_n(s)], \ \forall s\in\mathcal{S}$. The resulting linearly scalarized multi-objective optimization problem (LSMOP) is to solve
\be
\label{eq:self_cost}
\begin{split}
u(\w)&=\min_{s\in\mathcal{S}} c(s, \w).
\end{split}
\ee
 For any weight $\w\in\mathbb{R}_{\geq 0}^n$, solution $s\in\mathcal{S}$, and $\lambda\in\mathbb{R}_{>0}$, it holds that $c(s, \lambda \w)=\lambda c(s, \w)$ implying that a minimizer of $c(s, \w)$ also minimizes $c(s, \lambda\w)$. Further, if $\w=[0, 0, \dots, 0]$, then $u(\w)$ is trivially 0. Thus, given $\w=[w_1,\dots,w_n]$ where not all elements are identically 0, and letting $\lambda = \left(\sum_{i=1}^nw_i\right)^{-1}$, we can obtain all non-trivial optimal solutions $u(\w)$ for all $\w\in\mathbb{R}^n_{\geq0}$ using weights $\w\in \W$ where
\be
\label{eq:Wdef}
\W=\Big\{\w\in\mathbb{R}^n_{\geq 0}, \sum_{i=1}^nw_i=1\Big\}.
\ee
We refer to the set $\mathcal{W}$ as the \emph{weight space}, and we use the notation
$$
s^*(\w)=\argmin_{s\in\mathcal{S}} \w\cdot \f(s), \ \forall \w\in \W,
$$
implying that $u(\w)=c(s^*(\w), \w)$ by \eqref{eq:self_cost}. We make the following assumptions:
\begin{assumption}[Exact Solution]
\label{ass:cstsexst}
%  For any weight $\w\in\mathbb{R}^n_{\geq 0}$, the optimization problem in \eqref{eq:self_cost} has at least one solution $S^*(\w)$ which can be determined.
An exact solver exists for the optimization problem~\eqref{eq:self_cost}.
\end{assumption}
\begin{assumption}[Bounded Objectives]
\label{ass:bndcost}
For any weight $\w\in\mathbb{R}^n_{\geq 0}$, and any solution $s^*(\w)$, the objectives $\f(s^*(\w))$ are bounded.
\end{assumption}
In this work, we propose a method to compute a finite set of weights $\Omega\subset \W$ that will, for any $\w^*\in\W$, allow us to approximate $u(\w^*)$ with $u(\w')$ for an appropriately chosen $\w'\in\Omega$. To evaluate the quality of a candidate set $\Omega$, we use the notion of regret from \cite{wilde2020active, wilde2021learning}, defined formally here:
\begin{definition}[Regret]
Given two weights $\w',\w^*\in\W$, the regret
of $\w'$ under $\w^*$ is defined as
\be
\label{eq:AbsRegret}
\rAbs(\w' |\w^*) = \w^*\cdot \f\left(s^*(\w')\right) - u(\w^*).
\ee
\end{definition}
Intuitively, $r(\w'|\w^*)$ represents the \blue{\emph{error}} in cost incurred by using an optimal solution given weight $\w'$ (given by $s^*(\w')$) to approximate a solution given weight $\w^*$. We now formally state the main problem addressed in this work.

\begin{problem}[Min-Max Regret Sampling]
\label{prob:minmax}
For the LSMOP~\eqref{eq:self_cost} and some integer $K>0$, find a set of sampled weights $\Omega$ that solves
\be
\label{eq:objective}
\begin{aligned}
\min_{\Omega}& \max_{\wuser\in \W} \min_{\w'\in \Omega}
r(\w' |\wuser)\\
s.t.\;& |\Omega|\leq K.
\end{aligned}
\ee
\end{problem}
Given a weight $\w^*\in\W$ and a set $\Omega\subset\W$, the first minimization in \eqref{eq:objective},  $\min_{\w'\in\Omega} r(\w'|\w^*)$ represents the sub-optimality of approximating a solution $s^*(\w^*)$ with a solution $s^*(\w')$ where $\w'$ is the weight in $\Omega$ that minimizes this sub-optimality -- i.e., $\w'$ is a best representative of $\w^*$ in $\Omega$. The maximization $\max_{\w^*\in\W}\min_{\w'\in\Omega}r(\w'|\w^*)$, represents the regret of the worst represented weight $\w^*\in\W$ by elements in $\Omega$. We refer to the solution of this maximization as the \emph{maximum regret given} $\Omega$. In total, \eqref{eq:objective} seeks a set $\Omega$ such that the regret of the worst represented element in $\W$ is minimized. 

In this paper, we offer an approximate solution to the optimization in \eqref{eq:objective} by way of an algorithm that computes a feasible solution $\Omega$ such that the maximum regret given $\Omega$ is bounded. In the next section, we provide the theoretical groundwork that makes this solution possible.

%------------------------------------------------
\section{Problem Analysis}
\label{sec:analysis}
%------------------------------------------------
We begin with a structural analysis of the cost function $u(\w)$ to derive an efficient algorithm for solving Problem \ref{prob:minmax}.
First, we make two critical observations.
\begin{observation}
\label{obs:self_opt_traj}
Given any two weights $\w^*, \w' \in \W$, we have
\begin{equation}
\label{eq:optisopt}
u(\w^*)\leq \w^*\cdot\f(s^*(\w')) , 
\end{equation}
 That is an optimal solution given weights $\w^*$ will incur no higher cost than a solution that is optimal for some different weight vector $\w'$. Here, $s^*(\w')$ is optimal given weights $\w'$ but not necessarily optimal given weights $\w^*$. By \eqref{eq:AbsRegret}, the inequality in \eqref{eq:optisopt} implies that $\rAbs(\w'|\w^*)\geq 0$.
\end{observation}
\blue{
\begin{observation}[Optimal Cost Concavity]
\label{obs:concavity}
The optimal cost function $u(\w)$ is a concave function of $\w$.  Indeed, for each $s\in\mathcal{S}$, the cost $c(s, \w)=\w\cdot \f(s)$ is an affine function of $\w$ (and is therefore concave). Therefore, $u(\w)=\min_{s\in \mathcal{S}}c(s, \w)$ is concave \cite{boyd2004convex}[Section 3.2.3].
\end{observation}
Observation \ref{obs:concavity} motivates the following Corollary:
}
\blue{
%---------------------------------------------
\begin{corollary}[Optimal Cost Continuity]
\label{lem:cost_cont}
Under Assumptions \ref{ass:cstsexst}, \ref{ass:bndcost}, the function $u(\w)$ is continuous on the interior of $\W$. Further, $u(\w)$ is continuous on the boundary of $\W$ in the direction of its interior.
\end{corollary}
\begin{proof}
By Observation \ref{obs:concavity}, $u(\w)$ is concave in $\W$. Noting in addition that $\W\subset \mathbb{R}^n$ is convex, it must hold that $u(\w)$ is continuous on the interior of $\W$. This is because concave functions are continuous on the interior of convex sets. Next, consider two weights $\w', \w''\in\W$ one of which lies on the boundary of $\W$ and one in the interior of $\W$. Suppose that $||\w'-\w''||\leq \delta$ for some $\delta>0$ arbitrarily small, and -- without loss of generality -- that $u(\w')>u(\w'')$. Because it is assumed that $u(\w)$ is bounded on $\W$, if it experiences a discontinuity on the line connecting $\w', \w''$, it must hold that $u(\w)$ experiences a jump between $\w', \w''$. That is, there must exist a $M\in\mathbb{R}_{>0}$ independent of $\delta$ such that $u(\w'') + M\leq  u(\w')$. Since $||\w''-\w'||\leq \delta$, and $f(s^*(\w''))$ is bounded by Assumption \ref{ass:bndcost}, there must exist a value of $\delta$ sufficiently small so as to guarantee that $(\w'-\w'')\cdot f(s^*(\w'')) < \nicefrac{M}{2}$. Therefore, by construction,
\begin{equation*}
    \begin{split}
        \w'\cdot f(s^*(\w'')) < \frac{M}{2}+\w''\cdot f(s^*(\w''))=\frac{M}{2}+u(\w'')\\
        \leq \frac{M}{2} + u(\w') - M = u(\w') - \frac{M}{2}<u(\w').
    \end{split}
\end{equation*}
Therefore, $\rAbs(\w''|\w')=\w'\cdot f(s(\w''))-u(\w')<0$ which is a contradiction by Observation \ref{obs:self_opt_traj}.
\end{proof}
}
%%%%%%%%%%%%%%%%%%%%%%%%%%%%%%

Critically, the previous results do not require unique solutions $s^*(\w)$ or continuous objectives $\f(s^*(\w))$. Extending these results: 
\begin{theorem}[Convexity of Regret]
\label{thm:convex_regret}
For a fixed weight $\w'\in\W$, the regret $\rAbs(\w'|\w)$ is a convex function of $\w$.
\end{theorem}
\begin{proof}
By the definition of regret in \eqref{eq:AbsRegret},
$
\rAbs(\w'|\w)=\w\cdot\f(s^*(\w'))-u(\w).
$
where $\w\cdot\f(s^*(\w'))$ is linear in $\w$ and $u(\w)$ is concave by Observation \ref{obs:concavity}. Thus, $\rAbs(\w'|\w)$ is the difference of a linear and concave function of $\w$ which is convex.
\end{proof}
Because $u(\w)$ is continuous and concave, \blue{it must hold that the function lies below any sub-gradient.} % it must be differentiable almost everywhere \cite{poznyak2009advanced}.
This motivates the following Corollary \blue{which follows directly from the definition of $u(\w)$ and the concavity of $c(s, \w)$ for each fixed $s\in\mathcal{S}$ \cite{boyd2004convex}[Section 6.5.5]. }
%%%%%%%%%%%%%%%%%%%%%%%%%%%%%%
\begin{corollary}[Sub-gradient Optimal Cost]
\label{cor:derivative}
For any $\w\in\W$, the sub-gradients of $u(\w)$ are given by
$$
\partial u(\w)=\f(s^*(\w)),
$$
for any minimizing solution $s^*(\w)\in\mathcal{S}$.
\end{corollary} 
\blue{Observe that sub-gradients are defined even for non-differentiable continuous functions.} Further, The results above imply that given two weights $\w', \w^*\in\W$, the regret $r(\w'|\w^*)$ --- which coincides with the error incurred by approximating a solution $s^*(\w^*)$ with the solution $s^*(\w')$ --- is exactly the error of approximating a concave function via \emph{linear interpolation}. Indeed, the first order approximation of $u(\w^*)$ given $u(\w')$ is given by
$
u(\w^*)\approx u(\w') + \nabla u(\w')\cdot(\w^*-\w')
$
assuming $u$ is differentiable at $\w'$.  However, by Corollary \ref{cor:derivative}, $\nabla u(\w)=\partial u(\w)=\f(s^*(\w))$. This together with the definition $u(\w')=\f(s^*(\w'))\cdot \w'$ allows us to conclude that $u(\w^*)\approx u(\w')+\f(s^*(\w'))\cdot \w^* - \f(s^*(\w'))\cdot \w'=\f(s^*(\w'))\cdot \w^*$. The error of this first order approximation is therefore $\f(s^*(\w'))\cdot \w^* - u(\w^*)$ which is exactly the regret $r(\w'|\w^*)$. 

This is illustrated in Figure \ref{fig:Thm2} (a). Further, given any two weights $\w', \w''\in \W$, the maximum regret on the line segment $L$ between $\w', \w''$ given $\Omega=\{\w', \w''\}$ occurs at the weight on $L$ coinciding with the intersection of the tangent lines to $u(\w)$ at $\w'$ and $\w''$ along $L$ (Figure \ref{fig:Thm2} (b)). In light of this analysis, the objective in \eqref{eq:objective} is solved by a set $\Omega$ that provides the best linear interpolation of the concave function $u(\w)$. These insights are leveraged in the following section.
\begin{figure*}[t!]
    \centering
    \begin{subfigure}[t]{0.45\textwidth}
        \centering
        \includegraphics[width=0.99\textwidth]{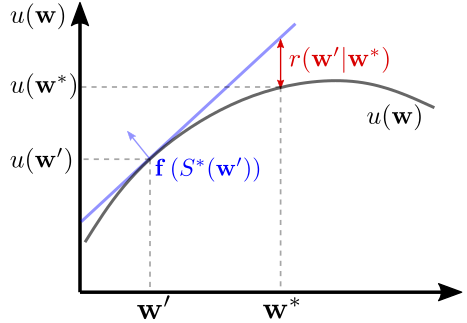}
        \caption{}
    \end{subfigure}%
    ~ 
    \begin{subfigure}[t]{0.45\textwidth}
        \centering
        \includegraphics[width=0.99\textwidth]{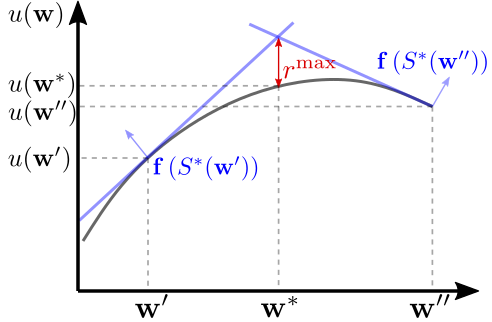}
        \caption{\bigskip}
    \end{subfigure}
    \caption{\footnotesize (a) Regret as the error of a first order approximation. (b) Maximum regret given a set $\Omega=\{\w', \w''\}$ at the intersection of their tangent lines.\normalsize}
    \label{fig:Thm2}
\end{figure*}

%------------------------------------------------
\section{Algorithm}
%------------------------------------------------
In this section, we present our solution to Problem \ref{prob:minmax}. The algorithm we propose works by recursively adding weights to a solution set $\Omega$. A strong candidate weight to add is one that is least represented by the current iteration of $\Omega$. The basic framework for such an approach could be described recursively:
\be
\label{eq:recurrel}
\Omega^{k+1}=\Omega^{k}\cup\{\argmax_{\w^*\in \W}\min_{\w'\in\Omega^k}r(\w'|\w^*)\},
\ee
where $\Omega^k$ is the solution after $k$ iterations from an initial set. Here, \eqref{eq:recurrel} recursively adds the weight with the maximum regret given $\Omega^k$. Obtaining the maximizer $\w^*$ is non-trivial due to its nested structure. Instead, our approach replaces $r(\w'|\w^*)$ in \eqref{eq:recurrel} with an upper bound $R(\w'|\w^*)$ whose maximizer $\w^*$ is obtained from a linear program (LP). 
\begin{figure*}[!t]
    \centering
    \begin{subfigure}[t]{0.44\textwidth}
        \centering
        \includegraphics[width=0.85\textwidth]{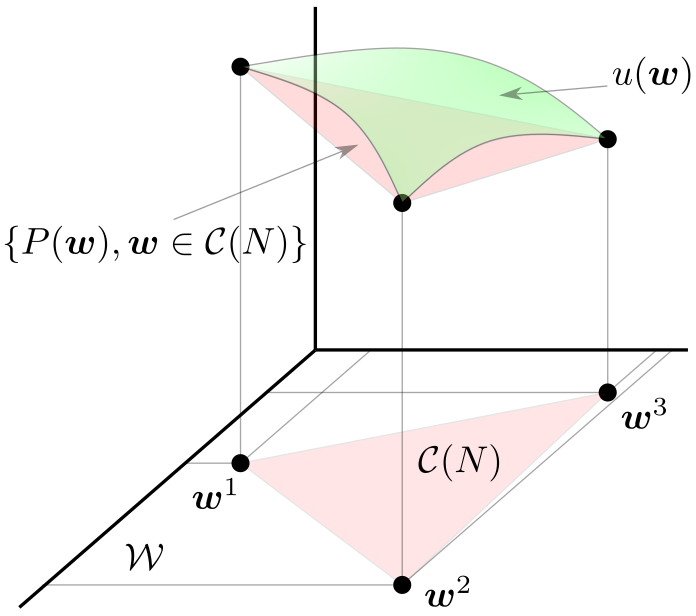}
        \caption{\footnotesize $P(\w), \C(N)$ for $N=\{\w^1,\w^2,\w^3\}$ \normalsize}
    \end{subfigure}
    ~ 
    \begin{subfigure}[t]{0.44\textwidth}
        \centering
        \includegraphics[width=0.85\textwidth]{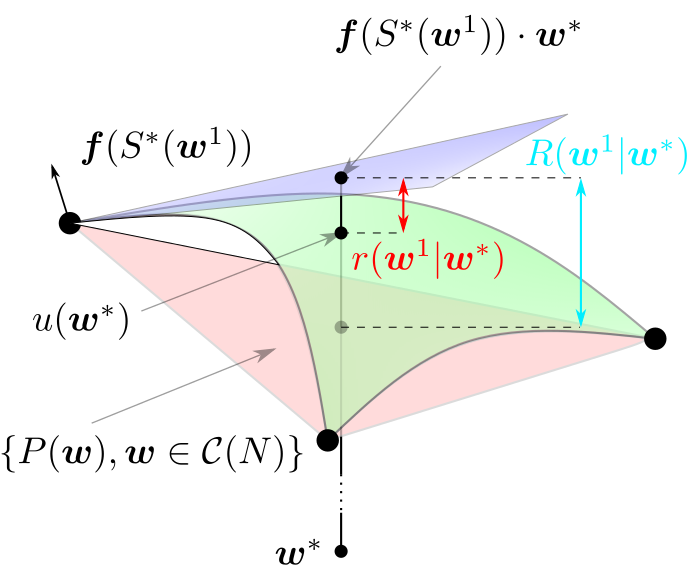}
        \caption{\footnotesize Regret $r,$ and regret bound $R$\normalsize\bigskip}
    \end{subfigure}
    \caption{\footnotesize Illustrative example of a neighborhood and its properties\normalsize.}
    \label{fig:SetsAndFunctions}
\end{figure*}
Given a set of weights $\Omega\subseteq\W$, we define $N$ as a set of $n$ \blue{(recall that $n$ is the number of objectives)} linearly independent weights $\w^1,\dots,\w^n\in\Omega$, and we let $\C(N)\subset\mathbb{R}^n$ denote the convex hull of $N$. We will loosely refer to $N$ as a neighborhood. 
\blue{We define a linear lower bound of the objective value $u(\w)$ from \eqref{eq:self_cost} inside a neighbourhood $N$. Let $P:\W\rightarrow \mathbb{R}_{\geq 0}$ be the linear function taking values $P(\w^i)=u(\w^i)$ for all $\w^i\in N$ (Figure \ref{fig:SetsAndFunctions} (a)).}
%For a neighborhood $N$ and $u(\w)$ defined in \eqref{eq:self_cost}, we let $P:\W\rightarrow \mathbb{R}_{\geq 0}$ be the linear function taking values $P(\w^i)=u(\w^i)$ for all $\w^i\in N$.
We denote the difference between the tangent plane at $\w'$ and the $P$ evaluated at $\w^*$ with $
R(\w'|\w^*) =\f(s^*(\w'))\w^*-P(\w^*), \ \forall \w'\in N, \w^*\in\C(N)$. Further, let
\begin{equation}
\label{eq:Rdef}
    \bar{R}(N)=\max_{\w^*\in\C(N)}\min_{\w'\in N} R(\w'|\w^*),\ \
    \bar{\w}(N)=\argmax_{\w^*\in\C(N)}\min_{\w'\in N} R(\w'|\w^*).
\end{equation}
Finally, we let $\mathcal{F}(N)$ represent the set of feature vectors of the neighborhood:
\begin{equation}
\label{eq:Fdef}
\mathcal{F}(N)=\{\f(s^*(\w^i)), \w^i\in N\}.
\end{equation}

Thus, $R(\w'|\w^*)$ is similar to $r(\w'|\w^*)$ from \eqref{eq:AbsRegret}, but with $u(\w^*)$ replaced with $P(\w^*)$. These definitions, illustrated in Figure \ref{fig:SetsAndFunctions}, motivate the Theorem:
\begin{theorem}[Upper Bound of Maximum Regret in a Neighborhood]
\label{thm:maxregret_bound}
Given a neighborhood $N$ of weights, it holds that
$$
\max_{\w^*\in \C(N)}\min_{\w'\in N}\rAbs(\w'|\w^*)\leq \bar{R}(N).
$$
\end{theorem}

\begin{proof}
For any fixed $\w^*\in \C(N)$, let $\w'_1=\argmin_{\w'\in N}\rAbs(\w'|\w^*)$ and let $\w'_2=\argmin_{\w'\in N}R(\w'|\w^*)$. Note that $\w'_1=\w'_2$. Indeed, we have $R(\w_2'|\w^*)\leq R(\hat{\w}|\w^*)$ for all $\hat{\w}\in N$ if and only if $\w^*\cdot\f(s^*(\w_2'))  \leq \w^*\cdot\f(s^*(\hat{\w}))$ which is equivalent to $r(\w_2'|\w^*)\leq r(\hat{\w}|\w^*)$.

By Observation \ref{obs:concavity}, $u(\w)$ is concave on $\C(N)$ implying that for any $\w\in \C(N)$, $u(\w)\geq P(\w)$ (see Figure \ref{fig:SetsAndFunctions}). Thus, by \eqref{eq:AbsRegret}, $r(\w'|\w^*)\leq R(\w'|\w^*)$ for any $\w^*\in \C(N)$. The result follows.

\end{proof}

The value of $\bar{R}(N)$ with corresponding weight $\bar{\w}(N)$ from \eqref{eq:Rdef} can be obtained by solving the following LP:
\begin{equation}
\label{eq:LPMaxRegret}
    \begin{split}
        \max_{x\in\mathbb{R}, \w\in\mathbb{R}^n}& x - P(\w)\\
        s.t. &\begin{bmatrix} 
f_1(s^*(\w^1))& \ldots &f_n(s^*(\w^1)) &-1\\ \vdots& \ddots &\vdots&\vdots \\ 
f_1(s^*(\w^n))& \ldots& f_n(s^*(\w^n)) &-1
\end{bmatrix}\begin{bmatrix}
\w\\
x
\end{bmatrix}\geq \begin{bmatrix}
0\\
\vdots\\
0
\end{bmatrix},\\
&\w\in\C(N).
    \end{split}
\end{equation}
If $(x^*, \w^*)$ solves \eqref{eq:LPMaxRegret}, the optimal cost is given by $x^*-P(\w^*)=\bar{R}(N)$, and $\w^*=\bar{\w}(N)$. Indeed, for any feasible $x, \w$, it holds that $x\leq \min_{\w^i\in N}\f(s(\w^i))\cdot\w$. Since $x$ is maximized, this will hold with equality for $x^*, \w^*$. Therefore, the cost of \eqref{eq:LPMaxRegret} is equivalent to $\max_{\w\in\C(N)}\min_{\w^i\in N}R(\w^i|\w)=\bar{R}(N)$. A detailed explanation of the implementation for Equation~\eqref{eq:LPMaxRegret} is provided in the supplementary materials.

In \eqref{eq:LPMaxRegret}, if $P(\w)$ is replaced with $u(\w)$, then the resulting  problem is solved by $x^*,\w^*$ if and only if $\w^*$ maximizes the regret in $\C(N)$ given the neighborhood $N$. This problem is not linear and would require solving the LSMOP in \eqref{eq:self_cost} potentially many times. Using the LP in \eqref{eq:LPMaxRegret} our method is summarized in Algorithm \ref{alg:greedy} described in the next section.  We iteratively partition $\W$ into smaller neighborhoods, adding weights that result in the largest upper bound of regret. 

\begin{algorithm}[!t]	
	\DontPrintSemicolon
	\KwIn{An exact solver to find $s^*(\w), \f(s^*(\w))$; a budget $K$}
	\KwOut{Sampled weights $\Omega$ and maximum regret}
	 $\Omega \gets \{\vect{e}^i, i=1,\dots, n\}$ \grey{// where $e^i$ is the $i^{th}$ row of the $n\times n$ identity matrix}\\
	 Obtain $s^*(\vect{e}^i), \f(s^*(\e^i)), i=1,\dots, n$ from exact solver
	 \\
	 $\N \gets \{\Omega\}$\\
% 	 $\Omega.\mathtt{obj}\gets \mathcal{F}(\Omega)=\{\f(S^*(\vect{e}^i)), i=1,\dots,n\}$\\
    %  $(\Omega.\verb|reg|, \Omega.\mathtt{weight})\gets(\bar{R}(\Omega), \bar{\w}(\Omega))$  \grey{// From \eqref{eq:LPMaxRegret}}\\
     \For{$k=n$ to $K$}{
     $N =$ neighborhood in $\N$ with maximum $\bar{R}(N)$\\
     \If{$\bar{R}(N)=0$}{\textbf{break} \grey{// Terminate if maximum upper regret bound is 0}}
     $\Omega \gets \Omega \cup \{\bar{\w}(N)\}$\\
     Obtain $s^*(\bar{\w}(N)), \f(s^*(\bar{\w}(N)))$ from exact solver
	\\
    $\N=\N\setminus N$ \grey{// Remove max-regret neighborhood}\\
	\For{$\w^i $ in $N$}{
	$N^i\gets N \setminus \{\w^i\}\cup\{\bar{\w}(N)\}$
	\grey{// Replace $w^i$ with weight of the maximum regret bound}\\
	\If{$N^i$ is a neighborhood, i.e., its weights are lin.~independent}
	{$\N=\N\cup N^i$
	\\
	$\mathcal{F}(N^i)\gets \mathcal{F}(N) \setminus \{\f(s^*(\w^i))\}\cup\{\f(s^*(\bar{\w}(N)))\}$
	}
	}
     }
	\Return{$\Omega$ and the maximum value of $\bar{R}(N)$ over all $N\in \N$}
	\caption{\textsc{Min-Regret Pareto Sampling (MRPS)}}
	\label{alg:greedy}
\end{algorithm}

\subsection{Algorithm Description}
%------------------------------------------------

Algorithm \ref{alg:greedy} creates and maintains a set $\N$ of neighborhoods $N\subset\W$. Each $N\in\N$ is a set of weights $N=\{\w^1,\dots,\w^n\}$ where $\w^i\in\Omega, i=1,\dots,n$. 
%These neighborhoods have three properties, $\mathcal{F}(N)$ (see \eqref{eq:Fdef}), representing the set of feature vectors for each weight in $N$,  $\bar{R}(N)$ the maximum upper bound of regret in $N$, and $\bar{\w}(N)$ the weight in $\C(N)$ at which the maximum upper bound of regret occurs. 
We compute $\bar{R}(N)$ and $\bar{\w}(N)$ with the LP in \eqref{eq:LPMaxRegret} for $N$ using the set of objective vectors $\mathcal{F}(N)$.
The algorithm begins with a single neighborhood whose weights are the $n$ canonical basis elements of $\mathbb{R}^n$ (Line 1). The algorithm then iteratively selects the neighborhood $N$ in $\N$ with the largest upper bound of regret (Line 5), and adds its regret weight $\bar{\w}(N)$ to $\Omega$ (Line 8). It then splits and replaces $N$ with at most $n$ smaller neighborhoods (Lines 11-15) formed by iteratively replacing elements in $N$ with $\bar{\w}(N)$ (Line 12). Finally, the algorithm returns the set $\Omega$ as well as an upper bound on the regret given $\Omega$ (Line 16). Two steps of the algorithm are illustrated in Figure \ref{fig:AlgoFirstSteps}, starting with a single neighborhood $N_1$ in (a) which is then split around $\w^3=\bar{\w}(N_1)$ into two new neighborhoods $\N=\{N_2, N_3\}$ in (b). Since $\bar{R}(N_3)>\bar{R}(N_2)$, $N_3$ is split around $\w^5=\bar{\w}(N_3)$ in (c). Finally, in (d), the red area shows the regret given $\Omega$.

\blue{Observe that Algorithm \ref{alg:greedy} may be modified to compute a set $\Omega$ given a desired maximum regret $r_{\text{max}}>0$. This could be accomplished by replacing the input $K$ with $r_{\text{max}}$, and the stopping criteria in Line 4 with a while loop that runs until $\bar{R}\leq r_{\text{max}}$. Here, $\bar{R}$ represents the maximum regret over all neighborhoods $\bar{R}=\max_{N\in\N}\bar{R}(N)$ and can be maintained in the body of the loop. Since $\N$ forms a partition of $\W$, we are guaranteed that the regret of any weight given $\Omega$ is no more than $\bar{R}$ by Theorem \ref{thm:maxregret_bound}. Therefore, if Algorithm \ref{alg:greedy} terminates when $\bar{R}\leq r_{\text{max}}$, the desired maximum regret is achieved.}

\begin{figure*}[t]
    \centering
    \begin{subfigure}[t]{0.46\textwidth}
        \centering
        \includegraphics[width=0.99\textwidth]{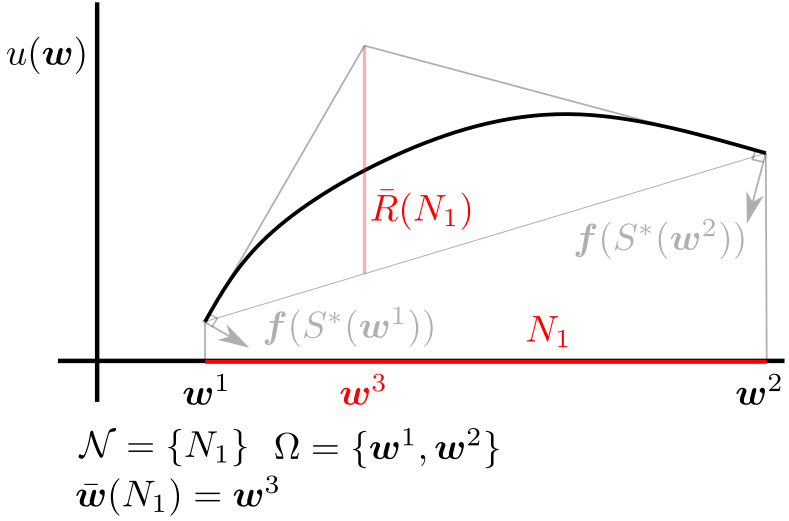}
        \caption{}
    \end{subfigure}%
    ~ 
    \begin{subfigure}[t]{0.46\textwidth}
        \centering
        \includegraphics[width=0.99\textwidth]{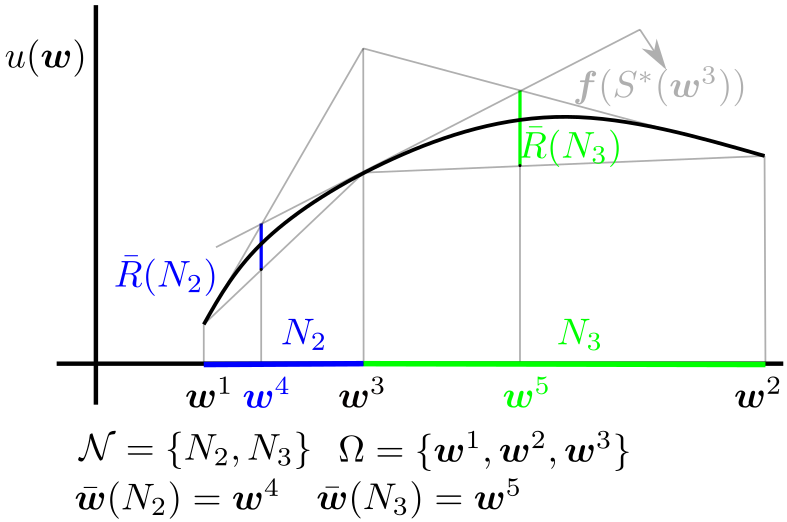}
        \caption{\bigskip}
    \end{subfigure}
    ~
    \begin{subfigure}[t]{0.46\textwidth}
        \centering
        \includegraphics[width=0.99\textwidth]{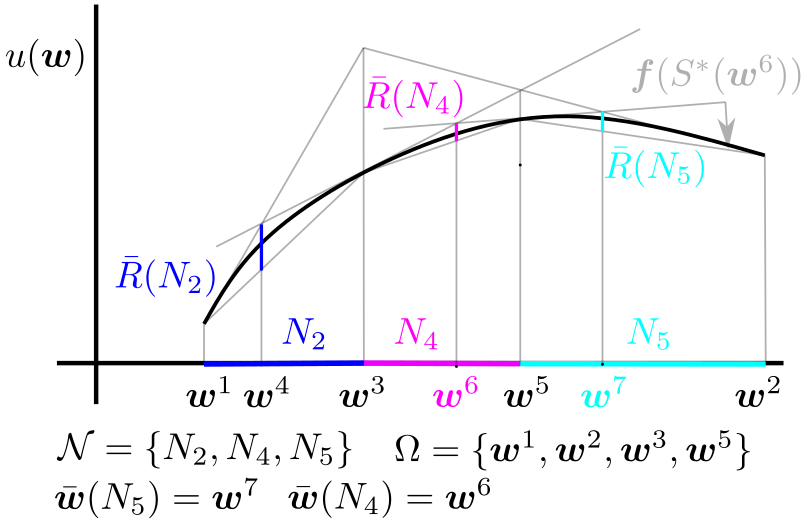}
        \caption{}
    \end{subfigure}
    ~
     \begin{subfigure}[t]{0.46\textwidth}
        \centering
        \includegraphics[width=0.99\textwidth]{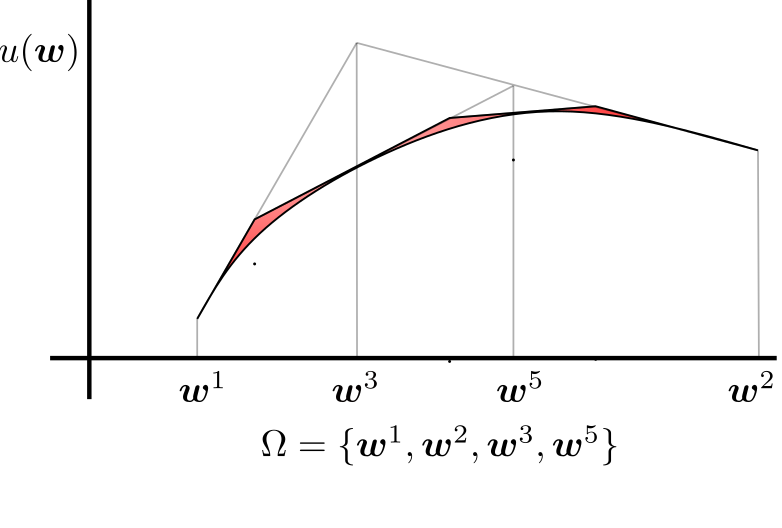}
        \caption{\bigskip}
    \end{subfigure}
    \caption{\footnotesize Illustration of the first two iterations of Algorithm \ref{alg:greedy}. \normalsize}
    \label{fig:AlgoFirstSteps}
\end{figure*}
%------------------------------------------------
\subsection{Algorithm Properties}
%------------------------------------------------
We observe several beneficial properties to the approach outlined above. First, for a budget of $K$, Algorithm \ref{alg:greedy} will require that the LSMOP in \eqref{eq:self_cost} be solved at most $K$ times, once per element of $\Omega$. Second, the value of $\bar{R}(N)$ returned by the algorithm is an upper-bound on the value of regret in the original problem~\eqref{eq:objective}. Indeed, initially $\N=\{\Omega\}$ and $\C(\Omega)=\W$. At every iteration, a neighborhood $N\in\N$ is split into at most $n$ sub-neighborhoods such whose convex hulls are disjoint and collectively form $\C(N)$. Then, by Theorem \ref{thm:maxregret_bound}, it holds that $\max_{\w^*\in \W}\min_{\w'\in \Omega}r(\w'|\w^*)\leq \max_{N\in\N}\bar{R}(N)$.

Further, the set $\Omega(K)$ returned by Algorithm \ref{alg:greedy} on input $K$ asymptotically and monotonically approaches a set with zero regret as $K\rightarrow\infty$. The proof of this is omitted for brevity, but holds because otherwise a neighborhood $N\in\N$ would fail to decrease in size when split (Lines 11-15). This in turns requires that there is a $\w^i\in N$ such that $||\bar{\w}(N)-\w^i||_2$ decreases to 0. Since $\bar{\w}(N)$ is chosen to maximize $\bar{R}(N)$, this can only occur if $\bar{R}(N)$ is unbounded at $\w^i$ implying that the objectives are unbounded at $\w^i$ in violation of Assumption \ref{ass:bndcost}. 
%Finally, in the case where the solution space $\mathcal{S}$ of the LSMOP in \eqref{eq:self_cost} is discrete, the function $u(\w)$ will be piece-wise linear by Lemma \ref{lem:derivative}. In this case, if $K$ is at least the number of linear pieces of $u(\w)$, the solution set $\Omega(K)$ has zero maximum regret and is the smallest set that accomplishes this. Indeed $\Omega(K)$ is comprised of exactly one weight in each linear piece of $u(\w)$. Since the regret is defined by the error of a first order approximation, this value is exactly zero on each linear piece. The next section illustrates the advantages of the approach. 

%------------------------------------------------
\section{Simulation Results}
%------------------------------------------------
We demonstrate our algorithm in simulations for two different domains: Trajectory planning and learning human preferences.
We compare our proposed algorithm to uniform sampling in the weight space \cite{zhang2007moea,wilde2020active,wilde2019bayesian}. This baseline is denoted by $\Uni$, while the proposed approach from Algorithm 1 is $\Greedy$.
%
%------------------------------------------------
\subsection{Experiment 1: Dubins Trajectories}
%------------------------------------------------
In the first experiment, we approximate the Pareto front for a LSMOP motion planning problem for a Dubins vehicle, as shown in the example from Figure \ref{fig:Dubins_traject_Pareto_examples}.
The objectives in the problem are trajectory length and integral of the squared (IS) jerk.
To find solutions $s^*(\w)$, the motion planner can compute different Dubins trajectories using different turn radia and pick the optimum among these.

%------------------------------------------------
\paragraph{Illustrative Example}
%------------------------------------------------
First, we present more insight into the example in Figure \ref{fig:Dubins_traject_Pareto_examples} with $K=7$ samples.
We notice that $\Greedy$ produces a larger variety of sample trajectories, especially those with shorter length. This is also visualized in the approximations of the Pareto front: $\Uni$ exhibits a large gap, while the proposed method places samples more homogeneously. In Figure \ref{fig:dubins_u_approx} we show the optimal cost $u(\w)$ (ground truth computed with 10,000 uniform weights), together with the tangent planes of the approximating samples of both approaches. Here, $\Greedy$ places more samples at the right end where $u(\w)$ changes more rapidly. Thus, the gap between the best approximating weight and the cost, i.e., the regret, is substantially smaller than $\Uni$. 

\begin{figure*}[!t]
    \centering
    \begin{subfigure}[t]{0.44\textwidth}
        \centering
        \includegraphics[width=0.99\textwidth]{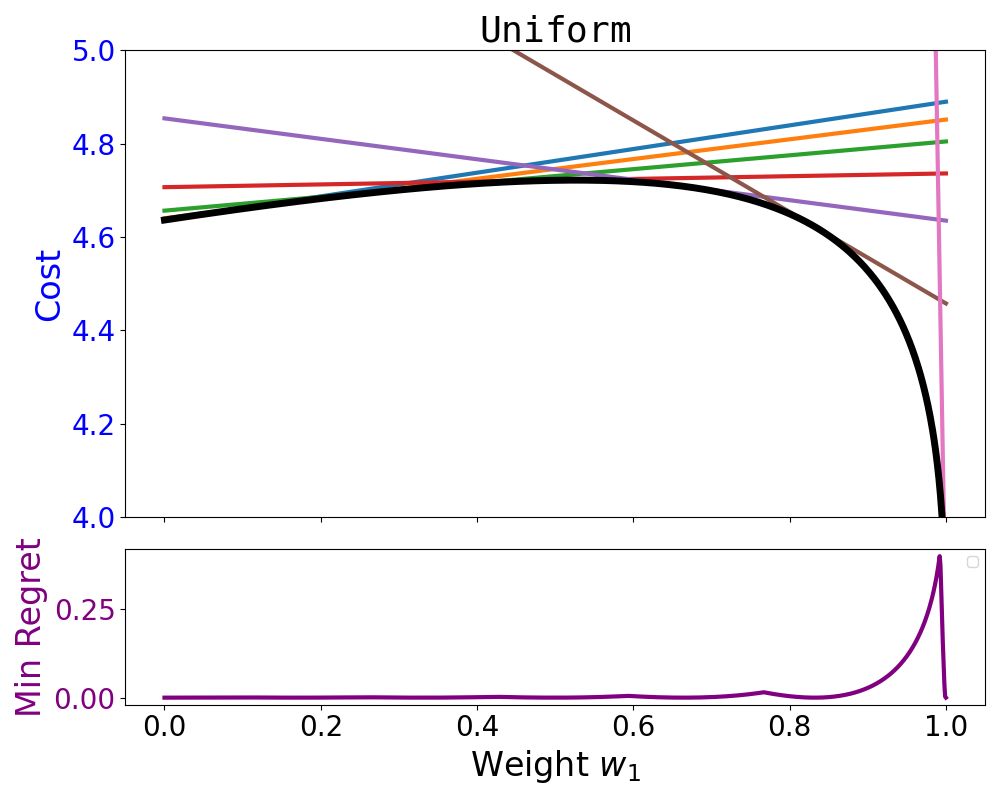}
        \caption{\footnotesize Approximation using $\Uni$.\bigskip\normalsize}
    \end{subfigure}~
    \begin{subfigure}[t]{0.44\textwidth}
        \centering
        \includegraphics[width=0.99\textwidth]{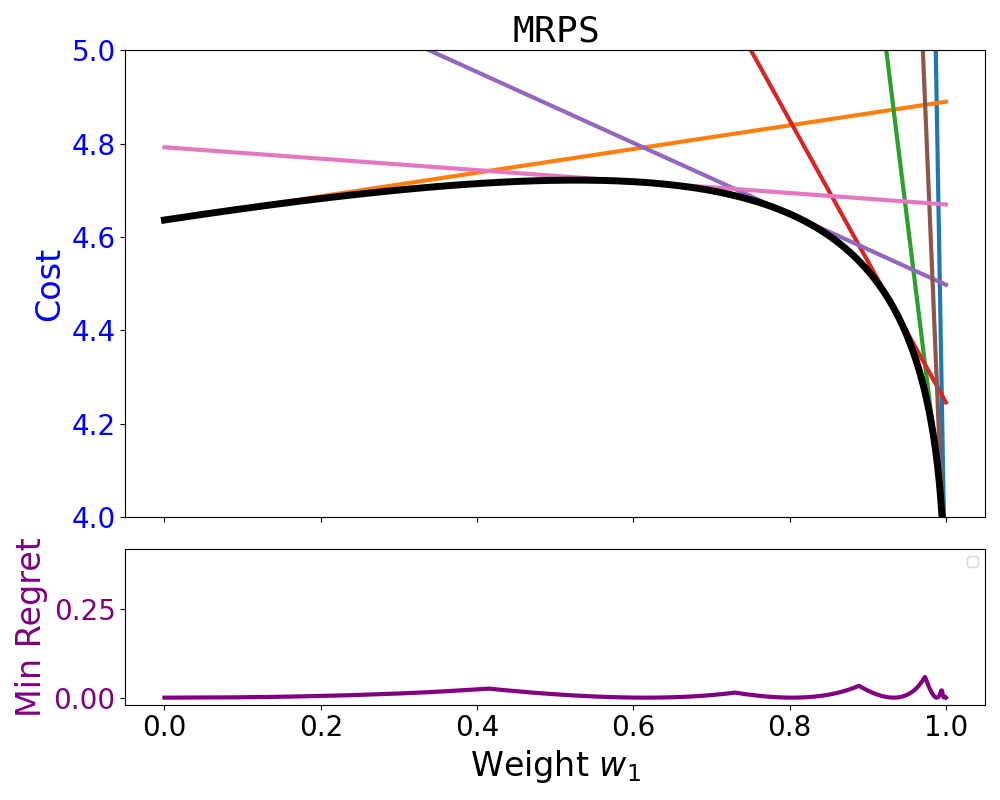}
        \caption{\footnotesize Approximation using $\Greedy$.\normalsize}
    \end{subfigure}
    \caption{\footnotesize Approximation of $u(\w)$ and resulting regret with $\Uni$ (a) and $\Greedy$ (b).}
    \label{fig:dubins_u_approx}
\end{figure*}

%------------------------------------------------
\paragraph{Quantitative Analysis}
%------------------------------------------------

We repeat the above Dubins planning experiment, but with randomized goal locations and various sampling budgets $K$.
We evaluate algorithm performance on three measures: the regret as defined in \eqref{eq:AbsRegret}, the relative regret where the difference in \eqref{eq:AbsRegret} is replaced with a ratio, and the hypervolume of the estimated Pareto front, similar to \cite{parisi2017manifold, xu2021multi}, using an exact integration for 2 features, and a Monte Carlo approximation as in \cite{parisi2017manifold} for higher dimensions.

Moreover, we consider cases with varying numbers of objectives. Figure \ref{fig:dubins3D} shows the result for a three objective system consisting of trajectory length, IS jerk and maximum jerk. When $K=0$ both approaches only use the basis solutions $e^1,\dots,e^n$, i.e., only the single objective solutions are available. For both, absolute and relative regret, we observe that $\Greedy$ achieves substantially smaller values for all $k>0$. With just $3$ samples, $\Greedy$ achieves a median regret of $.17$, corresponding to a median of $1.08$ in relative regret, i.e., percent error. In contrast $\Uni$ only achieves a medians of $.23$ and $1.16$ with $10$ samples for regret and relative regret, respectively.
The smaller estimated hypervolume also suggests that the  $\Greedy$-samples lead to a tighter linear approximation of the Pareto front.
In summary, the samples found by $\Greedy$ for small $K$ have smaller regret and hypervolume than samples found by $\Uni$ with $K=10$. Thus, our proposed method uses samples much more efficiently.
\begin{figure*}[t!]
    \centering
        \includegraphics[width=0.99\textwidth]{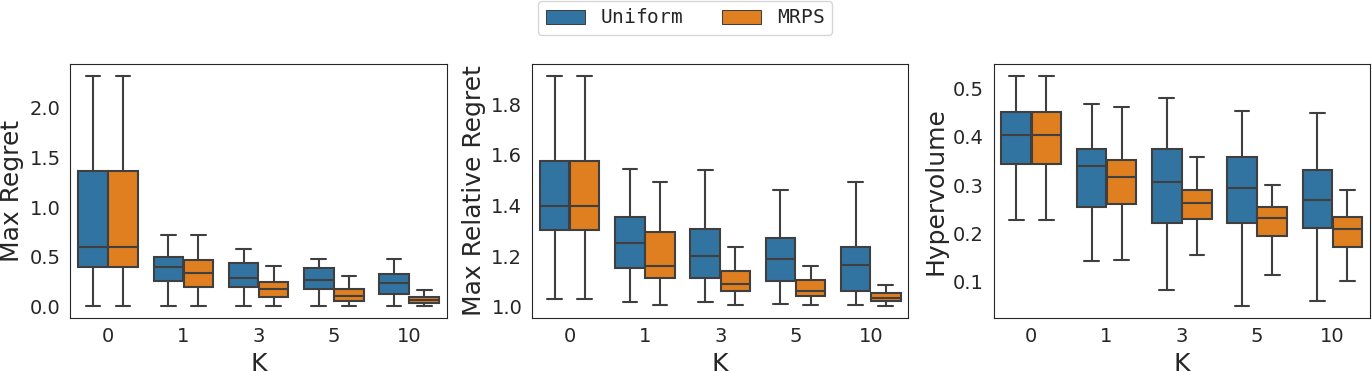}
        \caption{\footnotesize Numerical Results for the Dubins experiment with 3 objectives. \normalsize}
    \label{fig:dubins3D}
\end{figure*}

%------------------------------------------------
\paragraph{Varying number of objectives}
%------------------------------------------------
We also considered problem variances with $2$ or $4$ objectives. As expected, with two objectives (length and IS jerk) both approaches achieve a smaller regret with fewer samples. Nonetheless, the $\Greedy$ samples for $K=3$ are superior to the $\Uni$ samples for $K=10$, and the regret of $\Greedy$
actually converges to $0$ for $K=10$. We also tested the algorithm with four features, where the fourth objective is to avoid part of the environment. While the regret is larger for both approaches, we still observe the same trend that $\Greedy$ with $3$ samples is on-par with $\Uni$ using $10$ samples.

%------------------------------------------------
\subsection{Experiment 2: Reward Learning}
%------------------------------------------------
As a second example, we show how using the proposed sample method when learning user preferences, i.e., learning a user specific, but hidden weight vector $\w^*$. As mode of user interaction, we consider learning from choice \cite{sadigh2017active, biyik2019asking, wilde2020active, wilde2019bayesian, wilde2020improving} where the user is iteratively queried with two potential robot trajectories and indicates the preferred one. Repeating this over multiple iterations allows the robot infer about the user weights $\wuser$. Most algorithms for this problem require a set of presampled trajectories from which the best query is selected using some heuristic \cite{sadigh2017active, biyik2019asking, wilde2020active,wilde2019bayesian}. These presamples are either random trajectories \cite{sadigh2017active, biyik2019asking}, or optimal solutions for uniformly random weights for the LSMOP \cite{wilde2020active,wilde2019bayesian, wilde2021learningScale}.

Our proposed algorithm $\Greedy$ can be used to generate presamples for these learning problems. Thus, we compare the learning progress over $10$ iterations when either using $\Uni$ or $\Greedy$ samples. For a clear comparison, we use a simple, deterministic user model: presented with trajectories $A$ and $B$, they choose $A$ if and only if $f(A)\wuser\leq f(B)\wuser$. Similar to \cite{wilde2021learning, wilde2020improving} the trajectory preferred in the previous iteration constitutes one of the two trajectories presented in the next. 
% The robot can \emph{actively} choose the second trajectory to be presented in the next iteration.
\blue{The robot employs \emph{active} learning: it can choose the second trajectory to be presented in the next iteration.}
We employ a random selection from the presampled set ($\Random$), or the minmax regret approach from \cite{wilde2020active} ($\Regret$).
Finally, generating random user weights $\wuser$ is not trivial: When $\wuser$ is drawn uniformly random, the sample set $\Uni$ comes from the same distribution, biasing the experiment. Thus, we randomly select users from the union of two sample sets $\Omega(\Uni)$ and $\Omega(\Greedy)$ each of size $K=20$. 
We evaluate learning performance using relative regret.
% We evaluate learning performance using the relative regret $\nicefrac{f(s^*(\w'))\cdot\wuser}{u(\wuser)}$ where $\w'$ is the expected user weight.

\begin{figure*}[t!]
    \centering
        \includegraphics[width=0.85\textwidth]{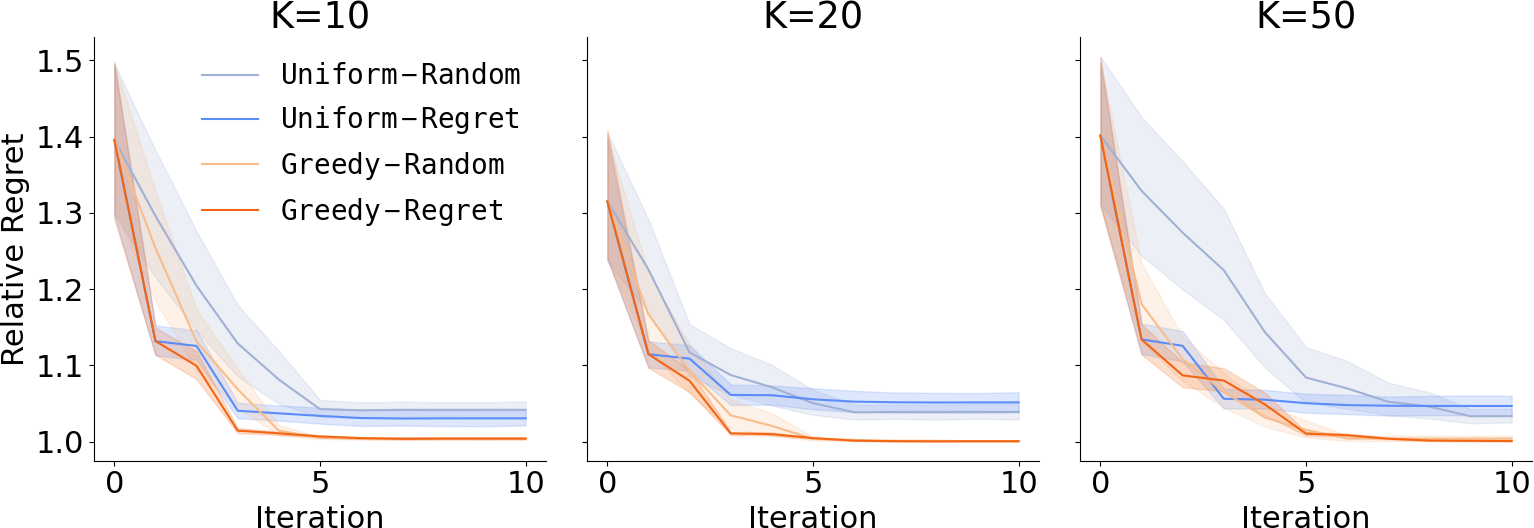}
    \caption{\footnotesize Learning from choice with  presampled sets of different size $K$. \blue{Blue shows learning progress when using uniformly randomly sampled trajectories for active learning with either random or regret queries. Orange shows learning progress for the same query methods, but trajectories are sampled with our proposed algorithm.} \normalsize}
    \label{fig:learning}
\end{figure*}

We test this learning framework using the four feature Dubins planning problem from earlier, with one fixed goal location. Figure \ref{fig:learning} shows the result for learning with presampled sets of various sizes $K$.
\blue{We compare four different approaches: $\Uni$ presamples combined with $\Random$ or $\Regret$ active learning, and $\Greedy$ presamples with $\Random$ or $\Regret$ active learning.}
Overall we observe that the $\Greedy$ samples lead to a smaller learning error than $\Uni$ samples, regardless of the query method ($\Random$ or $\Regret$). Indeed, the $\Uni$ sets do not always include a close-to-optimal sample such that the learning is eventually unable to make progress. That is the case when $\wuser$ was drawn from the $\Greedy$ samples. However, the opposite effect is negligible: Learning with the $\Greedy$ samples finds close-to-optimal solutions, implying that the error is very small even when $\wuser$ comes from $\Uni$.

When comparing different sizes $K$ of the sample sets, we observe that all approaches learn slightly slower for larger $K$ - the increased number of available trajectories seems to rather distract the learning algorithm than offering more informative queries. More surprisingly, when using $\Uni$ samples, the learning still stops making progress. This indicates that the larger set still does not contain close-to-optimal solutions. This further supports our earlier findings that while $\Greedy$ only needs small $K$ to find a close-to-optimal sample for any $\wuser$, while $\Uni$ is unable to achieve the same even for large $K$.

In summary, the experiment shows that using $\Greedy$ to generate pre-sampled solutions in reward learning allows for learning close-to-optimal solutions with significantly fewer samples than when relying on randomly generated pre-samples.
\iffalse
%------------------------------------------------
\subsection{Further results}
%------------------------------------------------
Finally, we also conducted an experiment for a multi-Traveling Salesman Problem (m-TSP) with two objectives: min-sum and min-max tour length. The results align with our other findings and further show the benefits of the proposed algorithm. For brevity, we refer the interested reader to the supplementary material.
\fi
%------------------------------------------------
\section{Discussion and Future Work}
%------------------------------------------------
In this work, we present a method by which a set of weights can be computed that bounds the maximum regret for any weight in the weight space. However, we only presented results for the case when an exact solver of the underlying LSMOP is available. In future, we will extend our work to the case when an approximation algorithm is used. Further, we constrained our analysis to linear scalarization. If the Pareto front of the underlying MOP is non-convex, linear scalarization fails to capture all Pareto-optimal solutions. 
% However, many of the theoretical results presented here can be extended to any scalarization technique that is concave in the weight space. 
\blue{However, many of the theoretical results presented here may be extended to Chebyshev scalarization which is known to be Pareto-complete.}

\bibliographystyle{IEEEtran}
% \bibliography{references}

\section{Appendix}
\appendix

%------------------------------------------------
\section{Implementation of Max-Min Neighbourhood Regret}
%------------------------------------------------

Below we detail the implementation of the linear program in Equation~(11).
We formulate the convex hull constraint $\w\in\C(N)$ using a scalars $\lambda^1,\dots, \lambda^n\in [0,1]$ to write $\w$ as a convex combination of the neighboughood weights $\w =\lambda^1\w^1 + \dots + \lambda^n\w^n $.
Thus, the equality constraints are given by
\begin{equation}
\label{eq:LPMaxRegret_ineq_details}
    \begin{bmatrix} 
1&& \ldots &1 &0& \ldots &0\\ 
0&& \ldots &0 &1& \ldots &1\\ 
-1&0& \ldots &0 &w_1^1& \ldots &w_1^n\\
0&-1& \ldots &0 &w_2^1& \ldots &w_2^n\\
\vdots&&\ddots&\vdots&\vdots&\ddots&\vdots\\
0&& \ldots &-1 &w_n^1& \ldots &w_n^n\\
\end{bmatrix}
\begin{bmatrix}
w_1\\
\vdots\\w_n\\
\lambda^1\\
\vdots\\\lambda^n
\end{bmatrix}
= \begin{bmatrix}
1\\1\\
0\\
\vdots\\
0
\end{bmatrix}.
\end{equation}
The first row ensures that $\w$ is normalized, i.e., lies in $\W$ where all its components sum to $1$. The second ensures the same for $\lambda^1,\dots, \lambda^n$.
The other rows ensure that the $i$-th element of the vecotr $\w$ is a convex combination of the $i$-th component of all neighbourhood vectors $\w^1,\dots, \w^n$. In the objective function, we write $P(\w)$ using the same convex combination as
\be
x - \sum_{i=1}^n \lambda^i u(\w^i).
\ee

Finally, we require that $w_i\geq0$ and $\lambda^i\geq0$ for all $i=1,\dots,n$.

% \section{Discrete Optimization Problem}
% In this section, we focus on the discrete optimization problems, i.e., the set $\mathcal{S}$ is a discrete set. Consider a pareto-optimal solution $S$ for the optimization problem and $P(S) \subseteq \mathcal{W}$ be the set of weights $\w$ such that $S$ is the optimal solution. The following result establishes the convexity of the set $P(S)$.
% \begin{lemma}
% Given a pareto-optimal solution $S$, the set $P(S)$ is a convex set.
% \end{lemma}
% \begin{proof}
% Assume by contradiction that the set if not a convex set. Therefore, there exist two weights $\w_1, \w_2 \in P(S)$ and a $\lambda \in (0, 1)$ such that $\w' = \lambda \w_1 + (1 - \lambda)\w_2 \not \in P(S)$. Then we have,

% \begin{align*}
%     u(S^*(\w') | \w') &= \lambda u(S^*(\w') | \w_1) + (1 - \lambda) u(S^*(\w') | \w_2)\\&\geq \lambda u(S | \w_1) + (1 - \lambda) u(S | \w_2) 
%     \\& = u(S | \lambda \w_1 +  (1 - \lambda) \w_2) = u(S | \w').
% \end{align*}
% This is a contradiction.
% \end{proof}

%------------------------------------------------
\section{Simulation Results}
In this section, we evaluate the performance of the proposed algorithm on a NP-hard optimization problem, namely traveling salesman problem with multiple agents.

Given a tour $T$ in $G = (V, E, c)$, we denote the cost of the tour with $\mathrm{cost}(T)$. Then the variation of the multi-traveling salesman problem is defined as follows:

\begin{problem}[Multi Traveling Salesman Problem (mTSP)]
Given a graph $G = (V, E, c)$, a set of $m$ robots, the weights $w_1, w_2$, and a depot $d \in V$,  the objective is to find a set of $m$ tours, $T_1, \ldots, T_m$, starting at $d$ such that the linear combination of the total tour length and the maximum tour length is minimized, i.e.,
\[
    \min_{T_1, \ldots, T_m} w_1\sum_{i = 1}^{m} \mathrm{cost}(T_i) + w_2 \max_{i \in \{1, \ldots, m\}} \mathrm{cost}(T_i). 
\]
\end{problem}

%------------------------------------------------
% We now consider the discrete and NP-hard traveling salesman problem with multiple salesmen (mTSP). 
% The problem can be stated with different objectives: minimizing the sum of tour lengths and minimizing the maximum tour length. Given weights $w_1$ and $w_2$ trading off the two objectives, we find an optimal solution using an integer linear program (ILP).
% \red{Result in Figure \ref{fig:mTSP} for a small environment with 15 vertices and 10 robots.}

\begin{figure*}[t!]
    \centering
    
        \includegraphics[width=0.95\textwidth]{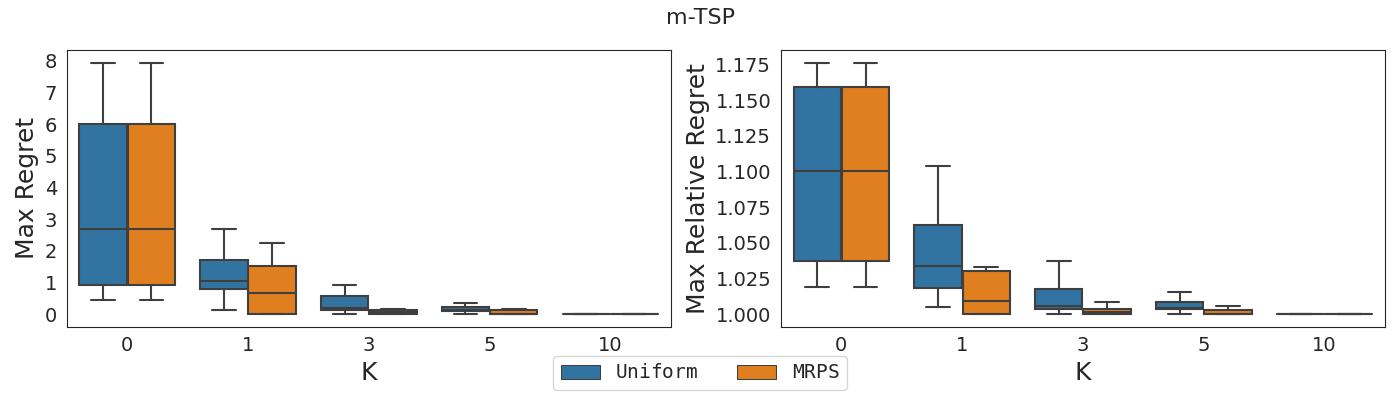}
        \caption{Sampling  trade-offs  between  min-sum  and  min-max  tour  lengths  for mTSP.}
    \label{fig:mTSP}
\end{figure*}

The two features, total tour length and maximum tour length, represent the total energy consumption of a fleet of robots to collectively service the tasks and the maximum time to service the tasks, respectively. Note that depending on the user-preferences or different scenarios, one of the features become more important. For instance, in a monitoring scenario with a set of robots, depending on the time and the state of the observed environment, the service time is more important than the total energy consumption of the fleet. Observe that finding the optimal solution of the mTSP problem for given $w_1, w_2$ is computationally expensive, therefore, changing the strategy based on the state of the environment in an online fashion is not feasible. Using the proposed approach, we generate  set of candidate sample weights that minimize the maximum regret for any possible scenario or user preference. Figure~\ref{fig:mTSP} shows the results of experiment with $15$ vertices to observe and $10$ robots. The left figure illustrates the maximum regret using the samples with the proposed method and the uniform sampling on the weight set. Note that the maximum regret of the proposed method almost converges to 0 with $3$ samples, while uniform sampling needs $10$ samples to achieve the same. Overall the difference between the methods is larger on the maximum relative regret measure.

\end{document}